\documentclass[10pt,letterpaper]{article}
\newcommand{\ours}{GradLev}
\usepackage{times}
\usepackage{amsmath,amssymb,amsthm,graphicx,booktabs,tabularx,xcolor,colortbl,tikz,float,placeins}
\usetikzlibrary{arrows.meta,positioning,calc,fit}
\usepackage{pgfplots}
\pgfplotsset{compat=1.18}
\usepackage{iclr2027_conference}
 
\newif\ifgradlevpreprint
\gradlevpreprinttrue
\ifgradlevpreprint
  \iclrfinalcopy
\fi
\usepackage{iftex}
\ifXeTeX
  \usepackage{fontspec}
\fi
\usepackage[colorlinks=true,linkcolor=accent,citecolor=accent,urlcolor=accent]{hyperref}
\definecolor{accent}{RGB}{24,95,99}
\hypersetup{pdftitle={\ours{}: Token-Parallel Test-Time Training via Costate Prediction}}
\ifgradlevpreprint
  \hypersetup{pdfauthor={Bo Liu and Qiang Liu}}
\else
  \hypersetup{pdfauthor={}}
\fi
\allowdisplaybreaks[1]

\newtheorem{theorem}{Theorem}
\newsavebox{\lmplotbox}
\newsavebox{\lmtablebox}
\newlength{\lmaxisheight}
\definecolor{primal}{HTML}{2878B5}
\definecolor{costate}{HTML}{7656A7}
\definecolor{prefill}{HTML}{C45B45}
\definecolor{deploy}{HTML}{237A78}
\definecolor{metricshade}{HTML}{DFE8F3}
\usepackage{caption}
\newcommand{\tableformat}{\small\renewcommand{\arraystretch}{1.23}\setlength{\tabcolsep}{4pt}}
\pgfplotsset{paperaxis/.style={
  scale only axis,width=5.35cm,
  axis lines*=left,axis line style={black!45,line width=.45pt},
  tick style={black!45,line width=.45pt},tick align=outside,
  major tick length=2pt,
  label style={font=\normalfont\fontsize{9}{10.5}\selectfont},
  tick label style={font=\normalfont\fontsize{8}{9}\selectfont},
  title style={at={(0,1)},anchor=south west,font=\normalfont\fontsize{9}{10.5}\selectfont,yshift=-8pt},
  xlabel style={yshift=2pt},ylabel style={yshift=-1pt},
  ymajorgrids=true,grid style={black!9,line width=.35pt},
  minor tick num=0,clip=true,
  /pgf/number format/1000 sep={}
},gradlevline/.style={color=primal,line width=1.35pt,no marks,line cap=round,line join=round},
  chunkline/.style={color=prefill,line width=1.2pt,no marks,dash pattern=on 3.6pt off 1.6pt,line cap=round,line join=round},
  staticline/.style={color=black!58,line width=1.25pt,no marks,densely dotted,line cap=round}}

\title{\ours{}: Token-Parallel Test-Time\\Training via Costate Prediction}
\author{Bo Liu \quad Qiang Liu\\
\normalfont The University of Texas at Austin\\
\normalfont\texttt{\{bliu,lqiang\}@cs.utexas.edu}}
\begin{document}
\maketitle
\ifgradlevpreprint\lhead{\small Preprint}\fi
\begin{abstract}
Test-time training (TTT) allows a model to improve its predictions at inference time by updating weights after every observed token. However, sequential gradient writes make parallel training difficult. We observe that, given layer inputs and activation gradients (costates), online gradient descent admits exact parallel scans for both forward evaluation and reverse backpropagation. \ours{} leverages this duality: a causal auxiliary network predicts costates across all tokens in parallel; associative scans compute the adapted weights and forward activations and propagate gradients backward; and the resulting gradient targets supervise the predictor via a consistency loss. Exact consistency guarantees exact recovery of the sequential online learner. At deployment, the auxiliary predictor is discarded, and the model updates natively via token-by-token forward and backward passes.
\end{abstract}

\vspace{-10pt}
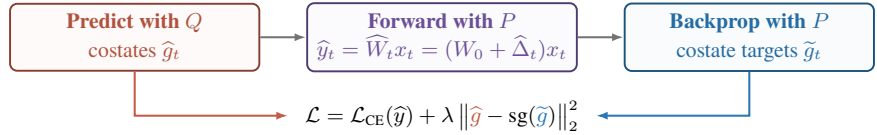
\begin{figure}[H]
\centering
\begin{minipage}{\dimexpr\textwidth-2\leftmargini\relax}

{\centering
\resizebox{\linewidth}{!}{%

\begin{tikzpicture}[
  font=\small,>={Latex[length=1.7mm,width=1.2mm]},
  flow/.style={->,line width=.85pt,draw=black!55},
  stage/.style={rounded corners=3pt,line width=.8pt,minimum height=.98cm,
  minimum width=3.65cm,align=center,inner sep=4pt}]
\node[stage,draw=prefill!68!black,fill=prefill!3,text=prefill!90!black] (q) at (1.6,1.5)
  {\textbf{Predict with $Q$}\\[2pt]costates $\widehat g_t$};
\node[stage,draw=costate!68!black,fill=costate!3,text=costate!90!black] (p) at (6.1,1.5)
  {\textbf{Forward with $P$}\\[2pt]$\widehat y_t=\widehat W_t x_t = (W_0 + \widehat \Delta_t) x_t$};
\node[stage,draw=primal!68!black,fill=primal!3,text=primal!90!black] (r) at (10.6,1.5)
  {\textbf{Backprop with $P$}\\[2pt]costate targets $\widetilde g_t$};
\draw[flow] (q)--(p);
\draw[flow] (p)--(r);
\node[align=center,fill=white,inner xsep=7pt] (loss) at (6.1,.35)
  {$\mathcal L=\mathcal L_{\mathrm{CE}}(\widehat y)
  +\lambda\,\bigl\|{\color{prefill}\widehat g}-
  \operatorname{sg}({\color{primal}\widetilde g})\bigr\|_2^2$};
\draw[flow,draw=prefill] (q.south)--(1.6,.35)--(loss.west);
\draw[flow,draw=primal] (r.south)--(10.6,.35)--(loss.east);
\end{tikzpicture}}\par}

\caption{A prefiller $Q$ predicts surrogate output costates $\widehat{g}_t$. Model $P$ accumulates these gradient updates via a parallel prefix scan to compute predictions $\widehat{y}_t = \widehat{W}_t x_t$, where $\widehat{W}_t$ denotes the hypothetical test-time trained weight matrix at time $t$. Then we backpropagate through $P$ to get the grounded costate targets $\widetilde{g}_t$. Because exact test-time training requires $\widehat{g}_t = \widetilde{g}_t$, we jointly train $P$ and $Q$ with a costate consistency loss in addition to the normal cross-entropy objective.}
\label{fig:main}
\end{minipage}
\end{figure}

\section{Introduction}

Every incoming token in language modeling serves a dual role: it is both a prediction target and an immediate supervisory signal. Rather than leaving internal representations static, an adaptive model can assimilate each new observation by updating temporary weights before predicting the next token. This online adaptation principle underlies dynamic evaluation and test-time training (TTT) \citep{dynamic,ttt}. By decoupling persistent cross-sequence knowledge (slow weights) from local within-sequence adaptations (fast weights) \citep{fastweights,ba}, test-time optimization provides a natural mechanism for tracking long-range document context, bursty entity bindings, and local distribution shifts. Meta-learning grounds this process by training the initialization and update rule so that rapid adaptation directly benefits future predictions \citep{maml,fastlm,tttlayers,e2e}.

The fundamental bottleneck is parallelism. Standard Transformer training evaluates all sequence positions concurrently via causal attention \citep{transformer}. In contrast, token-level test-time backpropagation introduces a strict serial dependency: the prediction at position $t$ requires the adapted weights $\theta_t$, which depend on the gradient at $t-1$, which in turn requires the unrolled activations and predictions from $t-1$. To bypass this sequential dependency, existing methods either coarsen updates into multi-token chunks \citep{tttlayers,inplace,lact}—sacrificing immediate reactivity—or restrict updates to simple associative recurrences found in linear attention and state-space models (SSMs) \citep{linear,gla,mamba2,longhorn}. Training an expressive model that performs a full forward pass, backward pass, and weight update at \emph{every single token} while retaining sequence-wide parallel training remains an open challenge.

In this work, we introduce \ours{}, a method that achieves fully parallel training for token-level test-time backpropagation through a \emph{lifted} optimization framework \citep{maglev}. Instead of sequentially computing gradient writes position by position, an auxiliary causal network $Q$ predicts the \emph{costates}—the gradients of the loss with respect to the outputs of the adapted layers—across all tokens simultaneously. Conditioned on these predicted costates and the layer inputs, online gradient descent decouples into an affine matrix recurrence. This structure allows the primary learner $P$ to evaluate its forward states, adapted predictions $\widehat{y}_t$, and transposed backward passes concurrently across the entire sequence via efficient parallel prefix scans.

To train the predictor without human annotations or serial unrolling, model $P$ retraces backpropagation through the scanned weights to generate target costates $\widetilde{g}_t$. The entire architecture is optimized end-to-end with a joint objective: task cross-entropy on the adapted outputs and a consistency loss that aligns predicted costates $\widehat{g}_t$ with their retraced targets $\widetilde{g}_t$. These targets are exact gradients on the proposed trajectory. Training stops their outer derivatives and uses the first-order rule described below. At zero consistency error, the parallel training trajectory matches the causal online trajectory. At test time, the predictor $Q$ is discardedand $P$ updates online by writing exact per-token gradients.

\section{Test-Time Backpropagation}
\label{sec:framework}
Test-time backpropagation uses each observed target to improve later predictions. For tokens $s_0,\ldots,s_T$, let $\theta_t$ be the parameters used to predict $s_{t+1}$ from the prefix $s_{\leq t}=(s_0,\ldots,s_t)$. Its negative log-probability is $\ell_t(\theta_t)$, with the prefix and historical cache implicit. An update rule $U$ adapts the model from its learned initialization $\theta_0$:
\begin{equation}
 \theta_{t+1}=U\!\left(\theta_t,\nabla_\theta\ell_t(\theta_t)\right),
 \qquad
 \min_{\theta_0}\;\mathbb E\!\left[\frac1T\sum_{t=0}^{T-1}\ell_t(\theta_t)\right].
 \label{eq:framework-update}
\end{equation}
The expectation is over sequences. Each target is scored before updating, and its test-time gradient holds earlier updates and cached activations fixed. Learning the initialization and update rule is a meta-learning problem \citep{maml,bilevel}, connecting dynamic evaluation and TTT \citep{dynamic,ttt} with adaptive language modeling \citep{fastlm,tttlayers,e2e}. In practice, one could select a subset of $\theta_0$ as the fast weights to adapt online. Slow weights of $\theta_0$, including learned coefficients of $U$, are trained by the outer objective and remain fixed at deployment.

\section{\ours{}}
\label{sec:construction}
\ours{} replaces sequential gradient computation during training with three stages: costate prediction with $Q$, forward evaluation with $P$, and backward reconstruction with $P$. Following Maglev's lifted training principle \citep{maglev}, $Q$ proposes quantities that $P$ computes itself at deployment; a consistency loss aligns the proposals with those computations.

\paragraph{Observation: known costates expose exact scans.}
Consider $y_t=W_tx_t$, with $x_t\in\mathbb R^d$, $y_t\in\mathbb R^m$, and slow initialization $W_0\in\mathbb R^{m\times d}$. The costate $g_t=\nabla_y\ell_t(y_t)$ differentiates through subsequent layers to the prediction loss. Within each reset segment, use step size $\eta_t\geq0$, retention $\alpha_t\in[0,1]$, and \textbf{centered weight decay}:
\begin{equation}
\begin{aligned}
 W_{t+1}&=W_0+\alpha_t(W_t-W_0)-\eta_tg_tx_t^\top,\\
 \Delta_t:=W_t-W_0
   &=-\sum_{i=0}^{t-1}\eta_i
      \left(\prod_{j=i+1}^{t-1}\alpha_j\right)g_ix_i^\top,
      \qquad \Delta_0=0.
\end{aligned}
\label{eq:general-expanded}
\end{equation}
Positions start at zero; empty sums are zero and empty products are one. Each past write is discounted by subsequent retentions, with decay centered at $W_0$.

\textbf{Given inputs, costates, and gates, gradient descent is an exact parallel scan.} Forward reads $W_tx_t$ and backward reads $W_t^\top g_t$ reuse these writes, so the predictor need only supply costate vectors.

\paragraph{1. Costate Prediction with $Q$.}
The prefiller predicts $\widehat g_t\in\mathbb R^m$ for each adapted map from causal features before the adapted subnetwork. It may observe the current target $s_{t+1}$: its proposal writes only to weights used at later positions. All proposals are computed in parallel across tokens.

\paragraph{2. Forward with $P$.}
Replace $g_i$ by $\widehat g_i$ in Equation~\ref{eq:general-expanded}. An exclusive scan, which includes only writes at $i<t$, computes
\begin{equation}
 \widehat\Delta_{t+1}=\alpha_t\widehat\Delta_t-\eta_t\widehat g_tx_t^\top,
 \qquad
 \widehat y_t=(W_0+\widehat\Delta_t)x_t.
 \label{eq:general-scan}
\end{equation}
Initialize $\widehat\Delta_0=0$ and read before each write. For multiple adapted layers, $x_t$ is each layer's input on this proposed trajectory. Layers run in forward order, with token positions evaluated in parallel within each layer. The negative sign belongs to the update; $\widehat\Delta_t$ is added to $W_0$.

\paragraph{3. Backward Reconstruction with $P$.}
Start at the prediction loss and work backward through the layers. Suppose the higher layers have supplied the output costate $\widetilde g_t=\partial\ell_t/\partial\widehat y_t$ at the current map. Holding its incoming weights and historical caches fixed, pass the gradient to its input:
\begin{equation}
 \widetilde g_t^{\,x}:=\frac{\partial\ell_t}{\partial x_t}
   =(W_0+\widehat\Delta_t)^\top\widetilde g_t
   =W_0^\top\widetilde g_t+\widehat\Delta_t^\top\widetilde g_t.
 \label{eq:general-reverse}
\end{equation}
The last term is a transpose scan over the \emph{same historical writes} $(x_i,\widehat g_i)$ used in the forward, queried by the incoming $\widetilde g_t$. If $x_t=\sigma(z_t)$ follows an elementwise activation, pass $\widetilde g_t^{\,z}=\sigma'(z_t)\odot\widetilde g_t^{\,x}$ to the preceding linear layer; $z_t$ is its preactivation and $\odot$ denotes elementwise multiplication. Repeating these steps gives every layer's exact costate on the proposed trajectory.

\paragraph{Training objective.}
Combine the prediction loss with a consistency loss that matches each proposal $\widehat g_t$ to its reconstructed costate $\widetilde g_t$:
\begin{equation}
 \mathcal L=\frac1T\sum_{t=0}^{T-1}\left[
 \ell_t(\widehat y_t)+\frac{\lambda}{m}
 \bigl\|\widehat g_t-\operatorname{sg}(\widetilde g_t)\bigr\|_2^2\right].
 \label{eq:general-objective}
\end{equation}
Here $\lambda\geq0$ weights consistency and $\operatorname{sg}$ is the stop-gradient operator. For multiple maps, average their dimension-normalized consistency losses. The three stages construct this objective; the optimizer's backward pass is separate. At deployment, remove $Q$ and update $P$ from its own costates after scoring each target.

\begin{theorem}[Causal consistency and exact recovery]
\label{thm:consistency}
Fix a finite input-target sequence and a deterministic causal network with acyclic within-token computation and finite local derivatives. Training and deployment share initialization, resets, causal gates, and update rules; reads precede writes, and the test-time reverse holds incoming weights and historical caches fixed. Let $G=(G_0,\ldots,G_{T-1})$ collect all proposed costates by token, and let $F(G)$ collect the reconstructed targets. Each target block $F_t$ depends only on $G_{<t}:=(G_0,\ldots,G_{t-1})$. Consequently, $F$ has a unique fixed point $G^\star$, the serial learner's costate trajectory:
\[
 \widehat G=F(\widehat G)\quad\Longleftrightarrow\quad
 \widehat G=G^\star.
\]
At this solution, the parallel and serial matrices, activations, and test-time gradients coincide in exact arithmetic.
\end{theorem}
\begin{proof}
An incoming matrix at token $t$ uses only writes at $i<t$, so the forward and reverse at $t$ depend only on $G_{<t}$. Thus $G_0^\star=F_0(\varnothing)$ and $G_t^\star=F_t(G_{<t}^\star)$ recursively determine the unique solution. The serial learner obeys this recursion and the same update rule.
\end{proof}
With fixed update coefficients and detached gradient writes, exact consistency also recovers the first-order model-agnostic meta-learning (MAML) slow-weight gradient for the same per-token objective \citep{maml}. Our implementation additionally learns the gates and retains their derivatives. The theorem guarantees exact test-time gradients, not the full meta-gradient. Nonzero consistency errors can be amplified (Appendix~\ref{sec:consistency-bound}).

\paragraph{Practical architecture and training.}
A learned affine gate $f:\mathbb R^d\to\mathbb R^2$ maps input $x_t\in\mathbb R^d$ to scalar logits. For matrix $j$ with input width $d_j$ and a fixed write-strength cap $\mu_{\max}>0$, use
\[
 (a_t,b_t)=f\!\left(\operatorname{sg}(x_t)\right),\qquad
 \alpha_t=\operatorname{sigmoid}(a_t),\qquad
 \eta_{j,t}=\frac{\mu_{\max}}{d_j}\operatorname{sigmoid}(b_t).
\]
In practice, one could share the retension $\alpha_t$ and learning rate $\eta_{j, t}$ across different layers from a shared input activation $x_t$.

Applying \ours{} on top of architectures with Attention requires special care. Let $q_t, k_t, v_t$ be the current query, key, and value, and $K_{<t},V_{<t}$ the past kv-caches. The test-time reverse differentiates the attention output
\[
 o_t=\operatorname{Attn}\!\left(q_t,
 [\operatorname{sg}(K_{<t});k_t],
 [\operatorname{sg}(V_{<t});v_t]\right).
\]
Here $\operatorname{Attn}$ is scaled dot-product attention and $[\,;\,]$ concatenates rows. Past cache entries are read unchanged. As a result, for the step-3 backpropagation through $P$, one should modify the reverse computation graph such that we treat the previous kv-cache as fixed constants.

\section{Related work}
\label{sec:related}
This section provides a brief overview of existing works related to \ours{} in online adaptation, fast weights, and parallel recurrent training.

\paragraph{Fast weights and learning to learn.}
Fast-weight programmers separate persistent parameters from temporary, input-dependent weights \citep{fastweights}; later work uses fast matrices to associate recent activations \citep{ba}. Learned optimizers parameterize the update rule \citep{learnopt}, while model-agnostic meta-learning (MAML) and bilevel optimization learn through adaptation \citep{maml,bilevel}. \ours{} learns the initialization and gates of an explicit gradient-descent rule. Related work interprets in-context learning as optimization: \citet{schlag} connect linear Transformers to fast-weight programming, and \citet{iclgd} study gradient descent within linear attention. \ours{} instead differentiates its next-token loss at \mbox{deployment to update selected weights.}

\paragraph{Test-time training and adaptive language models.}
Dynamic evaluation updates language models on observed context \citep{dynamic}, and TTT learns from test-instance supervision \citep{ttt}. Recent theory studies when test-time gradient steps help linear Transformers under task and distribution mismatch \citep{ttttheory}. Fast Weight Layers add a final adaptive module and express gradient writes evaluated at slow weights as linear attention \citep{fastlm}. TTT layers make a learned linear model or multilayer perceptron (MLP) the recurrent state \citep{tttlayers}. End-to-End TTT meta-trains next-token-loss adaptation \citep{e2e}, while In-Place TTT adapts existing MLP down-projections with an aligned objective \citep{inplace}. Large Chunk Test-Time Training (LaCT) uses large update chunks for hardware efficiency and nonlinear fast-weight capacity \citep{lact}. \ours{} parallelizes training through costate prediction while \mbox{preserving token-level updates at deployment.}

\paragraph{SSMs and memory as online learners.}
Longhorn treats SSM layers as meta-learned online learners. Pretraining learns input representations; inference updates associative memory with a diagonal approximation to an implicit regression update \citep{longhorn}. Titans adds neural memory with learned retention and momentum \citep{titans}. Miras organizes memory design by architecture, internal objective, retention, and optimization rule \citep{miras}. Atlas optimizes memory over current and past context \citep{atlas}, while MesaNet solves regularized regression using chunkwise conjugate gradients \citep{mesanet}. \ours{} uses the predictive loss to adapt selected maps and reconstructs their \mbox{forward and reverse with scans.}

\paragraph{Efficient matrix recurrences.}
Linear attention gives a recurrent matrix representation of attention \citep{linear}. The Structured State Space sequence model (S4) supports efficient long-range computation \citep{s4}, Mamba makes state transitions input-dependent \citep{mamba}, and state-space duality unifies efficient recurrent and chunkwise computation \citep{mamba2}. Mamba-3 further develops discretization, complex-valued state updates, and multi-input, multi-output states \citep{mamba3}. Gated linear attention (GLA) provides data-dependent forgetting with hardware-efficient training \citep{gla}; DeltaNet and Gated DeltaNet support error-correcting associative writes \citep{deltanet,gdn}. \ours{} reuses gated additive scans \citep{fla} with negative costate writes, rather \mbox{than Gated DeltaNet's associative-regression residual.}

\paragraph{Parallel training of nonlinear recurrence.}
Associative scans parallelize linear recurrences \citep{parallelrnn}. Parallel nonlinear equation solving, the nonlinear differential equation as fixed-point iteration (DEER) framework, and later stable solvers recover trajectories through iterative updates across positions \citep{song,deer,stable}. The Latent Recurrent Transformer (LRT) trains previous-token latent recurrence through parallel refinement \citep{lrt}. Other methods learn auxiliary trajectories: NextLat predicts latent transitions \citep{nextlat}, and Supervised Memory Training learns recurrent neural network (RNN) transitions from Transformer-generated predictive states \citep{smt}. Maglev uses a causal prefiller and consistency loss to train a decoder that produces its own memories at inference \citep{maglev}. \ours{} predicts costates, whose writes support forward and transpose scans. This relates to synthetic gradients \citep{synthetic}, but the predictions here are supervised by a full reverse and removed at deployment. Training constructs the trajectory in \mbox{three stages without iterative refinement.}

\section{Experiments}
\label{sec:experiments}
In this section, we provide two preliminary experiments to demonstrate the usefulness of \ours{}.
Section~\ref{sec:toy} asks whether the scans reproduce a token-level MLP learner given exact costates. Section~\ref{sec:lm} evaluates learned costate prediction through the deployment quality of a 300M language model, adapting only its final MLP.

\subsection{Recovering token-level TTT in an MLP}
\label{sec:toy}
This experiment checks whether parallel forward and backward scans reproduce the outputs and gradients of serial TTT. It tests the construction at exact consistency, independently of learning $Q$.

\paragraph{Experimental setup.}
Each sequence contains $T=32$ independent input--target pairs, with $x_t\sim\mathcal N(0,I_4)$ and $c_t\sim\mathcal N(0,I_3)$; $I_k$ is the $k\times k$ identity. A two-layer MLP predicts
\begin{equation}
 y_t=W_{2,t}\tanh(W_{1,t}x_t),\qquad
 \ell_t(y_t)=\tfrac12\|y_t-c_t\|_2^2,
 \label{eq:toy}
\end{equation}
where $W_{1,t}\in\mathbb R^{7\times4}$ and $W_{2,t}\in\mathbb R^{3\times7}$. Serial autograd updates both matrices after each token. Its costates are supplied to one parallel forward/backward construction, with the same initialization and update rule. A control replaces only the adapted transpose $W_{2,t}^\top$ in the backward by $W_{2,0}^\top$. Figure~\ref{fig:toy} reports maximum absolute discrepancies over coordinates, tokens, and three seeds in 64-bit floating-point arithmetic (FP64). Appendix~\ref{sec:toy-protocol} specifies initialization and gates.

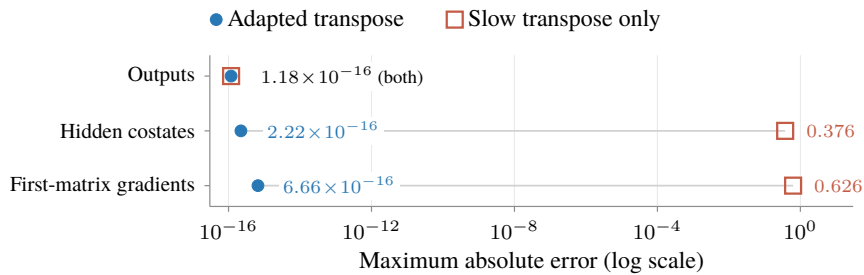
\begin{figure}[H]\centering
\begin{tikzpicture}
\begin{axis}[paperaxis,width=8.5cm,height=2.0cm,
  xmode=log,xmin=3e-17,xmax=30,xtick={1e-16,1e-12,1e-8,1e-4,1},
  ymin=-.38,ymax=2.38,ytick={0,1,2},
  yticklabels={First-matrix gradients,Hidden costates,Outputs},
  xlabel={Maximum absolute error (log scale)},ylabel={},
  ymajorgrids=false,xmajorgrids=true,
  legend style={at={(0,1.13)},anchor=south west,draw=none,
    font=\small,legend columns=2,inner sep=0pt,
    /tikz/every even column/.append style={column sep=14pt}},
  legend cell align=left,clip=false]
\addplot[black!18,line width=.65pt,forget plot] coordinates {(2.22044604925e-16,1) (.375650078709,1)};
\addplot[black!18,line width=.65pt,forget plot] coordinates {(6.66133814775e-16,0) (.626057138681,0)};
\addplot[only marks,primal,mark=*,mark size=2.2pt] coordinates {(1.17961196366e-16,2) (2.22044604925e-16,1) (6.66133814775e-16,0)};
\addlegendentry{Adapted transpose}
\addplot[only marks,prefill,mark=square,mark size=3.0pt,line width=.85pt] coordinates {(1.17961196366e-16,2) (.375650078709,1) (.626057138681,0)};
\addlegendentry{Slow transpose only}
\node[anchor=west,font=\scriptsize,fill=white,inner sep=1.5pt] at (axis cs:6e-16,2) {$1.18\!\times\!10^{-16}$ (both)};
\node[anchor=west,font=\scriptsize,text=primal,fill=white,inner sep=1.5pt] at (axis cs:9e-16,1) {$2.22\!\times\!10^{-16}$};
\node[anchor=west,font=\scriptsize,text=primal,fill=white,inner sep=1.5pt] at (axis cs:2.5e-15,0) {$6.66\!\times\!10^{-16}$};
\node[anchor=west,font=\scriptsize,text=prefill,fill=white,inner sep=1.5pt] at (axis cs:1.1,1) {$0.376$};
\node[anchor=west,font=\scriptsize,text=prefill,fill=white,inner sep=1.5pt] at (axis cs:1.8,0) {$0.626$};
\end{axis}
\end{tikzpicture}
\caption{\textbf{Exact forward and backward recovery.} With supplied exact costates, discrepancies from serial TTT remain at rounding precision. Replacing only the adapted transpose preserves the outputs but changes the reverse gradients. Errors are maxima over coordinates, tokens, and three seeds.}
\label{fig:toy}
\end{figure}

\paragraph{Results and interpretation.}
Maximum output, hidden-costate, and first-matrix gradient discrepancies are $1.18\times10^{-16}$, $2.22\times10^{-16}$, and $6.66\times10^{-16}$. Using only the slow transpose raises the last two to 0.3757 and 0.6261. Thus the scans recover serial MLP predictions and test-time gradients when proposals are exact. This verifies numerical equivalence; it does not measure learned costate-prediction accuracy or regression generalization.

\subsection{300M language modeling with final-MLP adaptation}
\label{sec:lm}
This study asks whether learning with $Q$ improves the deployed learner over a static Transformer and four-chunk TTT at the same token budget. \textbf{Only the final MLP is adapted} in both adaptive models.

\paragraph{Experimental setup.}
All three models have 21 Transformer layers, model width $d=1024$, two-matrix MLP width $m=4096$, 16 attention heads, a tied 32,768-token vocabulary, and context length 1024. Chunk-TTT updates every 256 tokens using detached, summed gradients; static has no test-time updates.

\ours{} uses a two-block causal Transformer $Q$ with two costate heads, conditioned on the MLP input and observed target embedding. Retention is $\alpha_t=1$, the write cap is $\mu_{\max}=4$, and write strength starts at 0.9; $Q$'s heads start at zero. Chunk-TTT uses the same gate family and initialization on fragment-mean inputs. Cross-entropy trains $P$ and the gates; consistency trains $Q$. Appendix~\ref{sec:protocol} gives the architecture and gradient stops.

All models share slow initialization, ClimbMix data order \citep{climb}, seed 42, a 262,144-token batch, and the $P$ optimizer schedule. Each trains for 112,420 steps (29.47B tokens), exceeding 100 tokens per non-embedding training parameter, including $Q$. A short screen selects half-rate $Q$, with $\lambda=1$. Training compute differs; Appendix~\ref{sec:protocol} gives budgets and tuning details.

Deployment evaluation removes $Q$, scores before updating, and resets at document/context boundaries. Metrics are validation negative log-likelihood (NLL), FineWeb-Edu bits per byte (BPB), LAMBADA word perplexity (PPL), and mean accuracy over Maglev's eight zero-shot tasks \citep{maglev}. Evaluation covers 65,536 validation tokens crossing chunk boundaries, 512 FineWeb-Edu documents \citep{fineweb}, and 27,072 task examples; Appendix~\ref{sec:downstream} gives details.

\begin{figure}[H]\centering
\captionsetup{skip=4pt}
\sbox{\lmtablebox}{%
\begin{minipage}{.57\linewidth}
\fontsize{8.5}{10.5}\selectfont
\renewcommand{\arraystretch}{1}\setlength{\tabcolsep}{2.5pt}
\newcommand{\lmmetric}[2]{{\renewcommand{\arraystretch}{.88}\begin{tabular}[c]{@{}l@{}}#1\\[-.5pt]{\fontsize{7.5}{9}\selectfont #2}\end{tabular}}}
\begin{tabular*}{\linewidth}{@{\extracolsep{\fill}}lrrr@{}}\toprule
Metric & Static & Chunk-TTT (4) & \ours{}\\\midrule
\lmmetric{Negative log-likelihood ($\downarrow$)}{Validation, nats/token}
 & 2.5733 & 2.5726 & \textbf{2.5707}\\\addlinespace[1pt]
\lmmetric{Bits per byte ($\downarrow$)}{FineWeb-Edu}
 & 0.7527 & 0.7530 & \textbf{0.7517}\\\addlinespace[1pt]
\lmmetric{Word perplexity ($\downarrow$)}{LAMBADA}
 & 24.66 & 24.54 & \textbf{23.56}\\\addlinespace[1pt]
\lmmetric{Zero-shot accuracy (\%) ($\uparrow$)}{Mean over eight tasks}
 & 50.77 & 50.81 & \textbf{50.99}\\
\bottomrule\end{tabular*}
\end{minipage}}
\setlength{\lmaxisheight}{2.2cm}
\newcommand{\lmplot}{%
\begin{tikzpicture}
\begin{axis}[paperaxis,width=3.8cm,height=\lmaxisheight,
  xlabel={Training steps (thousands)},ylabel={$\Delta\mathrm{NLL}$},
  xmin=0,xmax=115,xtick={0,50,100},
  ymin=-.019,ymax=.025,ytick={-.01,0,.01,.02},scaled y ticks=false,
  yticklabel style={/pgf/number format/fixed,/pgf/number format/precision=2},
  legend style={at={(0,1)},yshift=4pt,anchor=south west,draw=none,
    font=\fontsize{8}{9.5}\selectfont,legend columns=2,inner sep=0pt,
    /tikz/every even column/.append style={column sep=4pt}},legend cell align=left]
\addplot[black!35,densely dashed,line width=.6pt] coordinates {(0,0) (115,0)};
\addplot[staticline] coordinates {(2.048,-0.01538590342) (4.096,-0.00555177033) (8.192,-0.00540212542) (16.384,-0.00432087854) (32.768,-0.00435915962) (49.152,-0.00506940112) (65.536,-0.00271229073) (81.920,-0.00391671434) (98.304,-0.00238541886) (112.420,-0.00274385884)};
\addplot[chunkline] coordinates {(2.048,0.02165774256) (4.096,0.01182670891) (8.192,0.00242083520) (16.384,0.00148097426) (32.768,-0.00183742866) (34.816,-0.00335489213) (36.864,-0.00488639995) (38.912,-0.00124735013) (40.960,-0.00319331512) (43.008,-0.00539365783) (45.056,-0.00435817614) (47.104,-0.00528803468) (49.152,-0.00435410812) (51.200,-0.01251285151) (53.248,-0.00619864836) (55.296,-0.00271657109) (57.344,0.00160649046) (59.392,-0.00448065251) (61.440,0.00011101738) (63.488,-0.00369065255) (65.536,-0.00351126492) (67.584,-0.00096481293) (69.632,-0.00111426786) (71.680,-0.00275876373) (73.728,-0.00085667893) (75.776,-0.00293827429) (77.824,-0.00452214852) (79.872,-0.00436411425) (81.920,-0.00081847981) (83.968,-0.00352684036) (86.016,-0.00334544480) (88.064,-0.00158905238) (90.112,-0.00165835768) (92.160,-0.00342526287) (94.208,-0.00131703913) (96.256,-0.00013732910) (98.304,-0.00227257237) (100.352,-0.00164186582) (102.400,-0.00173895434) (104.448,-0.00284063444) (106.496,-0.00142465532) (108.544,-0.00185449421) (110.592,-0.00209880248) (112.420,-0.00192119554)};
\legend{,Static,Chunk-TTT (4)}
\end{axis}
\end{tikzpicture}}
\sbox{\lmplotbox}{\lmplot}
\addtolength{\lmaxisheight}{\dimexpr\ht\lmtablebox+\dp\lmtablebox-\ht\lmplotbox-\dp\lmplotbox\relax}
\sbox{\lmplotbox}{\lmplot}
\typeout{GRADLEV-PANELS: plot=\the\dimexpr\ht\lmplotbox+\dp\lmplotbox\relax; table=\the\dimexpr\ht\lmtablebox+\dp\lmtablebox\relax}
\begin{minipage}[t]{.40\linewidth}\vspace{0pt}\centering
\usebox{\lmplotbox}\par
\caption{\textbf{Deployment NLL gap.} \ours{} minus control; negative is better. Recorded GB300 scores, without smoothing.}
\label{fig:learning300}
\end{minipage}\hfill
\begin{minipage}[t]{.57\linewidth}\vspace{0pt}\centering
\usebox{\lmtablebox}\par
\captionsetup[table]{position=bottom,skip=4pt}
\captionof{table}{\textbf{300M final-MLP results.} B300 deployment after 29.47B training tokens. Bold marks best point estimates; task-level scores and uncertainty are in Appendix~\ref{sec:downstream}.}
\label{tab:main-results}
\end{minipage}
\end{figure}

\paragraph{Results and observations.}
\ours{} shows slight improvements in likelihood and mean zero-shot accuracy, probably due to the limited richness from adapting only the last MLP layer. Its final NLL is lower by 0.001856 than chunk-TTT and by 0.002617 than static; mean accuracy is 50.99\%, versus 50.81\% and 50.77\%, respectively. \ours{} leads chunk-TTT and static at every shared checkpoint from step 65,536 onward (Figure~\ref{fig:learning300}).

\section{Limitations and future work}
This work is a preliminary study of costate prediction as a mechanism for parallelizing token-level test-time training. Its primary contribution is the formulation of this approach; the empirical gains are modest, and the current experiments do not establish its effectiveness at larger scales. Demonstrating consistent benefits across model sizes, adaptation depths, and tasks remains an important direction for future work.

Several architectural questions also remain open. Extending adaptation to deeper networks requires a scalable predictor $Q$ that maintains accurate costate predictions across layers while preserving the computational benefits of parallel training. Sharing parameters or intermediate representations between $P$ and $Q$ may help reduce this overhead, but the resulting trade-offs between prediction accuracy, training cost, and optimization stability require further study. Another direction is to restrict gradient-based adaptation to a small set of parameters, which need not be full weight matrices. Determining which parameters to adapt, and how to design architectures that make effective use of these updates, could broaden the practical scope of \ours{}. Addressing these questions is necessary to assess whether the approach can support effective test-time learning in production-scale models.

\clearpage\appendix
\section{Worked example: an online MLP}
\label{sec:mlp-example}
This section derives the online MLP updates and their parallel forward and reverse computations.

Let $x_t\in\mathbb R^d$ be the input, $m$ the hidden width, and $\sigma$ an elementwise activation. The matrices $W_{1,t}\in\mathbb R^{m\times d}$ and $W_{2,t}\in\mathbb R^{d\times m}$ start at $W_{1,0}$ and $W_{2,0}$. Define the preactivation $z_t$, hidden activation $h_t$, and output $y_t$ by
\begin{equation}
 z_t=W_{1,t}x_t,\qquad h_t=\sigma(z_t),\qquad y_t=W_{2,t}h_t.
 \label{eq:online-forward}
\end{equation}
The costates $g_t^y=\nabla_y\ell_t(y_t)$ and $g_t^z=\nabla_z\ell_t(y_t)$ define the reverse and simultaneous writes:
\begin{align}
 g_t^z&=\sigma'(z_t)\odot W_{2,t}^\top g_t^y,
 \label{eq:online-reverse}\\
 W_{1,t+1}&=W_{1,0}+\alpha_t(W_{1,t}-W_{1,0})
              -\frac{\mu_t}{d}g_t^zx_t^\top,\nonumber\\
 W_{2,t+1}&=W_{2,0}+\alpha_t(W_{2,t}-W_{2,0})
              -\frac{\mu_t}{m}g_t^yh_t^\top.
 \label{eq:online-update}
\end{align}
Here $\odot$ is elementwise multiplication, $\alpha_t\in[0,1]$ is retention, and $\mu_t\geq0$ is the shared write coefficient, divided by each map's input width. Both gradients use pre-update weights.

Given costate proposals $\widehat g_t^z\in\mathbb R^m$ and $\widehat g_t^y\in\mathbb R^d$, the forward processes the layers in order and evaluates each recurrence by an exclusive token scan:
\begin{align}
 \widehat\Delta_{1,t+1}
   &=\alpha_t\widehat\Delta_{1,t}-\frac{\mu_t}{d}\widehat g_t^zx_t^\top,
 &\widehat z_t&=(W_{1,0}+\widehat\Delta_{1,t})x_t,\label{eq:scan1}\\
 \widehat\Delta_{2,t+1}
   &=\alpha_t\widehat\Delta_{2,t}-\frac{\mu_t}{m}\widehat g_t^y\widehat h_t^\top,
 &\widehat h_t&=\sigma(\widehat z_t),\label{eq:scan2}\\
 &&\widehat y_t&=(W_{2,0}+\widehat\Delta_{2,t})\widehat h_t.
 \label{eq:proposed}
\end{align}
Both corrections start at zero. The reverse supplies targets
\begin{equation}
\begin{aligned}
 \widetilde g_t^y&=\nabla_y\ell_t(\widehat y_t),\\
 \widetilde g_t^z&=\sigma'(\widehat z_t)\odot
 \left(W_{2,0}^\top\widetilde g_t^y+
              \widehat\Delta_{2,t}^\top\widetilde g_t^y\right).
\end{aligned}\label{eq:reconstructed-reverse}
\end{equation}
The transpose scan uses historical $\widehat g_i^y$ to reverse the same matrices as the forward. To reach earlier adapted layers, continue with $(W_{1,0}+\widehat\Delta_{1,t})^\top\widetilde g_t^z$; the final-MLP learner needs no further read.

The implemented MLP consistency loss is
\begin{equation}
\mathcal L_{\mathrm{cons}}=\frac1{2T}\sum_{t=0}^{T-1}\left[
 \frac{\|\widehat g_t^z-\operatorname{sg}(\widetilde g_t^z)\|_2^2}{m}
 +\frac{\|\widehat g_t^y-\operatorname{sg}(\widetilde g_t^y)\|_2^2}{d}
\right].\label{eq:consistency}
\end{equation}
Use $\lambda=1$ in Equation~\ref{eq:general-objective}. Deployment writes true costates after each scored target.

\paragraph{Language-model loss.}
\label{sec:normalization}
For the final MLP, let $u_t\in\mathbb R^d$ be the residual input and $E\in\mathbb R^{V\times d}$ the tied embedding/output matrix for vocabulary size $V$. With root mean square normalization (RMSNorm), the token loss is
\begin{equation}
 \ell_t(y)=-\log\left[\operatorname{softmax}\!\left(E\operatorname{RMSNorm}(u_t+y)\right)\right]_{s_{t+1}}.
 \label{eq:token-loss}
\end{equation}
Differentiate this loss through the head, normalization, and residual to obtain $g_t^y$, before batch or sequence averaging. The Gaussian error linear unit (GELU) derivative and adapted transpose then give $g_t^z$ by Equation~\ref{eq:reconstructed-reverse}.

\section{Synthetic MLP protocol}
\label{sec:toy-protocol}
This section specifies the numerical diagnostic in Section~\ref{sec:toy}. For seeds 17, 42, and 123, initialize the slow matrices with independent centered Gaussian entries of standard deviation $0.5/\sqrt4$ for $W_{1,0}$ and $0.5/\sqrt7$ for $W_{2,0}$. Independent vectors $v_\alpha,v_\mu\in\mathbb R^4$ have $\mathcal N(0,0.25^2)$ entries. Set $\alpha_t=\operatorname{sigmoid}(x_t^\top v_\alpha+1.5)$, $\mu_t=\operatorname{sigmoid}(x_t^\top v_\mu-0.2)$, and matrix step sizes $\mu_t/4$ and $\mu_t/7$. Corrections start at zero; both matrices update simultaneously with centered decay.

Independent serial autograd supplies the exact costates to a single parallel reconstruction using explicit causal sums in FP64. No serial hidden states are reused. The control changes only the backward transpose, so its output error is identical to the full construction. Errors in Figure~\ref{fig:toy} are maxima over coordinates, 32 token positions, and three seeds. This is a numerical consistency check, with no learned $Q$ or refinement loop.

\section{Forward and transpose scan reads}
\label{sec:scans}
This section makes explicit how the forward and backward reuse the same historical writes. For current query vectors $a_t\in\mathbb R^d$ and $b_t\in\mathbb R^m$, Equation~\ref{eq:general-expanded} gives
\begin{equation}
\begin{aligned}
 \widehat\Delta_t a_t
 &=-\sum_{i<t}\eta_i\left(\prod_{j=i+1}^{t-1}\alpha_j\right)
       \widehat g_i(x_i^\top a_t),\\
 \widehat\Delta_t^\top b_t
 &=-\sum_{i<t}\eta_i\left(\prod_{j=i+1}^{t-1}\alpha_j\right)
       x_i(\widehat g_i^\top b_t).
\end{aligned}\label{eq:scan-reads}
\end{equation}
The forward queries with $a_t=x_t$; the backward queries with $b_t=\widetilde g_t$. Historical values remain the proposals $\widehat g_i$ in both cases. Each read excludes the current write and all writes before the most recent document or context reset.

\section{What approximate consistency guarantees}
\label{sec:consistency-bound}
This section distinguishes exact recovery from stability to prediction errors. Theorem~\ref{thm:consistency} requires no contraction, but causality alone cannot bound trajectory error by a small consistency residual. For example, consider two scalar token blocks and the causal target map $F(G_0,G_1)=(0,\kappa G_0)$, where $\kappa\in\mathbb R$. Its unique fixed point is $G^\star=(0,0)$. For any $\varepsilon\in\mathbb R$, the proposal $\widehat G=(\varepsilon,\kappa\varepsilon)$ satisfies
\[
 \|\widehat G-F(\widehat G)\|_2=|\varepsilon|,
 \qquad
 \|\widehat G-G^\star\|_2=\sqrt{1+\kappa^2}\,|\varepsilon|.
\]
Thus an arbitrarily small residual can coexist with a large trajectory error when amplification is large. This example does not assert that the trained model is unstable; it explains why exact consistency alone does not establish accuracy at nonzero residuals.

\section{Language-model training details}
\label{sec:protocol}
This section records the architecture, training recipe, and numerical checks for the 300M experiments.

\paragraph{Architecture and Q conditioning.}
The backbone uses RMSNorm \citep{rmsnorm}, rotary positions, tanh-approximate GELU, a tied embedding/output matrix $E\in\mathbb R^{32768\times d}$, and context length 1024. The hidden and MLP widths are $(d,m)=(1024,4096)$. Let $x_t\in\mathbb R^d$ be the normalized input to the final MLP. Q receives $\operatorname{sg}(x_t)+\operatorname{RMSNorm}(\operatorname{sg}(E_{s_{t+1}}))$, using unit scale and stabilizer $10^{-6}$ for the target normalization. Its two causal Transformer blocks, final normalization, and two linear heads produce vectors $a_t^z\in\mathbb R^m$ and $a_t^y\in\mathbb R^d$. The proposals are
\begin{equation}
 \widehat g_t^z=\sigma'\!\left(\operatorname{sg}(W_{1,0}x_t)\right)\odot a_t^z,
 \qquad \widehat g_t^y=a_t^y.
\end{equation}
Compute $W_{1,0}x_t$ once and reuse it in P. Q is causal and document-isolated; target-conditioned proposals affect only later predictions.

\paragraph{Initialization and gradients.}
For fixed retention, a learned vector $w\in\mathbb R^d$ and scalar bias $b$ define the write strength and the two MLP step sizes:
\begin{equation}
 \mu_t=4\operatorname{sigmoid}\!\left(w^\top\operatorname{sg}(x_t)+b\right),
 \qquad \eta_{1,t}=\mu_t/d,\quad\eta_{2,t}=\mu_t/m.
 \label{eq:step-gate}
\end{equation}
The gate starts at $\mu_t=0.9$ with $w=0$ and $b=\log(0.9/3.1)$. Retention is one. Zero-initialized $Q$ heads make the first parallel forward static; native deployment already updates. The derivative-conditioning factor above is detached. Historical write inputs, scan costates, $Q$ inputs, and reverse targets are detached; read and gate derivatives remain.

\begin{table}[!ht]\centering
\caption{\textbf{300M final-MLP parameter and token budgets.} Counts exclude the tied embedding/output matrix. Training counts include Q and determine the reported token-to-parameter ratios.}
\label{tab:budgets}
\tableformat\begin{tabular*}{\linewidth}{@{\extracolsep{\fill}}lrrrr@{}}\toprule
Model & Deployed (M) & Training (M) & Tokens (B) & Tokens / param.\\\midrule
Static & 264.2852 & 264.2852 & 29.4702 & $111.5092\times$\\
Chunk-TTT (4) & 264.2862 & 264.2862 & 29.4702 & $111.5088\times$\\
\ours{} & 264.2862 & 294.7000 & 29.4702 & $100.0008\times$\\
\bottomrule\end{tabular*}
\end{table}

\paragraph{Optimization and packing.}
Each update averages 262,144 scored tokens: 16 graphics processing units (GPUs), two 1024-token sequences per GPU, and eight accumulation steps. Adam with decoupled weight decay (AdamW) uses peak rate $3\times10^{-4}$, coefficients $(0.9,0.95)$, epsilon $10^{-8}$, matrix decay 0.1, zero vector decay, and joint gradient-norm clipping at one. After 128 warmup steps, cosine decay reaches 10\% of the peak at step 112,420. \ours{} uses half-rate Q and consistency weight $\lambda=1$. The $Q$ multiplier applies to its optimizer update and AdamW decay; optimizer decay is separate from centered test-time decay.

Pack documents from rank-strided parquet row groups; reset fast state at document and packed-row boundaries. Hold out the final sorted shard.

\paragraph{Chunk comparator.}
A preliminary short screen selected four fixed 256-token chunks, unit retention, and the capped write gate initialized to 0.9. The gate uses the stopped mean MLP input of the current final-document fragment. At each boundary, both matrices receive the fragment's summed gradients, normalized by matrix input width and evaluated at the same incoming weights. Document resets clear memory, so only the final document can write into the next chunk. Costates and write inputs are detached; gates and later reads retain outer derivatives. Only \ours{} uses \mbox{a prefiller and consistency objective.}

\paragraph{Numerical implementation.}
Flash Linear Attention (FLA) provides the training and recurrent scans \citep{fla,gla}. Writes are shifted to implement exclusive reads and masked at document boundaries. Autocasting uses the 16-bit brain floating-point format (BF16), with 32-bit floating-point (FP32) parameters and fast states. The two MLP states require $2md$ FP32 entries per sequence, or 32 mebibytes (MiB) at this scale. Curves use GB300 evaluations; final scores use B300 replays. Their endpoint NLLs differ by at most $1.02\times10^{-4}$, so the two sources are reported separately.

\paragraph{Prefiller-rate selection.}
The screen compares $Q/P$ learning-rate ratios $1/4$, $1/2$, and $3/4$ using mean deployment NLL at steps 1024 and 2048. Their means are 3.571559, 3.571861, and 3.581692. The prespecified rule selects half-rate $Q$ when it is within 0.001 of the best score. All prefixes use the full 112,420-step schedule; only the selected setting continues. Discarded prefixes are additional tuning compute, and no full-budget unit-rate $Q$ control is available.

\section{Downstream evaluation}
\label{sec:downstream}
This section defines the deployment metrics and evaluation protocol for Table~\ref{tab:main-results}.

\begin{table}[!ht]\centering
\caption{\textbf{300M final-MLP results after 29.47B training tokens.} All models use the same B300 evaluation runtime. Metrics retain their original precision; bold marks the best value in each column.}
\label{tab:downstream-full}
\fontsize{8.5}{10.5}\selectfont
\renewcommand{\arraystretch}{1.18}\setlength{\tabcolsep}{2pt}
\begin{tabular*}{\linewidth}{@{\extracolsep{\fill}}l*{12}{r}@{}}\toprule
& \multicolumn{3}{c}{\textbf{Likelihood} ($\downarrow$)} & \multicolumn{9}{c}{Zero-shot accuracy (\%) ($\uparrow$)}\\
\cmidrule(lr){2-4}\cmidrule(l){5-13}
Model & \shortstack{\textbf{Valid.}\\\textbf{NLL}} & \shortstack{\textbf{FineWeb}\\\textbf{BPB}} & \shortstack{\textbf{LAMB.}\\\textbf{PPL}} & LAMB. & PIQA & Hella. & Wino. & ARC-E & ARC-C & SIQA & BoolQ & \textbf{Mean}\\\midrule
Static & 2.573312 & 0.752708 & 24.661 & 37.69 & \textbf{72.52} & \textbf{47.84} & \textbf{53.99} & 65.49 & \textbf{33.45} & \textbf{39.87} & 55.35 & 50.77\\
Chunk-TTT (4) & 2.572551 & 0.752980 & 24.538 & 37.40 & 72.09 & 47.73 & 53.91 & 65.19 & 31.83 & 38.18 & \textbf{60.18} & 50.81\\
\ours{} & \textbf{2.570695} & \textbf{0.751658} & \textbf{23.564} & \textbf{38.04} & 72.14 & 47.75 & 53.28 & \textbf{65.53} & 32.00 & 39.41 & 59.82 & \textbf{50.99}\\
\bottomrule\end{tabular*}
\par\vspace{4pt}
\begin{minipage}{\linewidth}\fontsize{8}{9.5}\selectfont
NLL uses the primary validation set; FineWeb-Edu BPB and LAMBADA word PPL use separate corpora. LAMBADA (LAMB.); Physical Interaction: Question Answering (PIQA); HellaSwag (Hella.); WinoGrande (Wino.); AI2 Reasoning Challenge (ARC), Easy (-E) and Challenge (-C); SocialIQA (SIQA). Mean averages eight task accuracies (27,072 examples per model). HellaSwag and ARC-C use length-normalized scores.
\end{minipage}
\end{table}

\paragraph{Experimental setup.}
The suite follows Maglev's task selection \citep{maglev}: LAMBADA \citep{lambada}, PIQA \citep{piqa}, HellaSwag \citep{hellaswag}, WinoGrande \citep{winogrande}, ARC-Easy and ARC-Challenge \citep{arc}, SocialIQA \citep{socialiqa}, and BoolQ \citep{boolq}. The mean equally weights eight task accuracies.

Each example starts with zero fast state and a fresh cache. Score each token before updating. Candidates use independent prompt-state copies; gold labels are consulted only after scoring. Choice scores sum continuation log-probabilities, normalized by token count for HellaSwag and ARC-Challenge. WinoGrande uses cloze completion. LAMBADA requires greedy agreement with every reference final-word token, without checking the following word boundary.

Chunk boundaries remain at 256 tokens. Many benchmark requests end before an update affects scoring, limiting evidence about repeated chunk adaptation; 1024-token NLL evaluations cross actual boundaries.

\paragraph{Results and uncertainty.}
Task-level outcomes in Table~\ref{tab:downstream-full} are mixed. Paired 95\% intervals for \ours{}'s mean-accuracy difference are $[-0.454,0.808]$ percentage points against chunk-TTT and $[-0.421,0.862]$ against static. Both include zero and omit training-seed variability. Validation and FineWeb-Edu scores were used during development, so neither provides fresh confirmation.

\paragraph{300M likelihood metrics and uncertainty.}
Primary NLL averages 65,536 validation tokens. FineWeb-Edu uses 512 documents, 563,229 scored tokens, and 2,764,147 original bytes in 8-bit Unicode Transformation Format (UTF-8); bits per byte divides total negative log-likelihood by the byte count times $\log 2$. Non-overlapping blocks contain at most 1024 targets, start with fresh states and a context-only beginning-of-sequence (BOS) token, and score every text token once. Full-document exact and normalized training matches are excluded; partial or semantic overlap may remain. LAMBADA word perplexity exponentiates the mean final-word loss over 5153 examples, summing subword losses before averaging.

Accuracy intervals use 10,000 paired bootstrap draws (seed 20260917), resampling examples within each task and equally averaging paired task differences. Limits are the 2.5th and 97.5th percentiles. They exclude training-seed and recipe-selection uncertainty.

\end{document}